%% file: iclr2027_conference.tex
\documentclass{article} 
\usepackage{iclr2027_conference,times}

\iclrfinalcopy

\input{math_commands.tex}

\usepackage{booktabs}
\usepackage[table]{xcolor}
\usepackage{multirow}
\usepackage{pgfplots}
\pgfplotsset{compat=1.18}
\usepgfplotslibrary{groupplots}
\usetikzlibrary{arrows.meta}
\definecolor{rowtint}{gray}{0.94}
\definecolor{rulegray}{gray}{0.55}
\definecolor{terracotta}{HTML}{C17F5A}
\definecolor{terracottalight}{HTML}{FBE3C1}
\definecolor{terracottadark}{HTML}{854C38}
\usepackage{hyperref}
\newcommand{\ours}[1]{\cellcolor{rowtint}#1}
\newcommand{\oursrow}{\rowcolor{rowtint}}

\usepackage{amsfonts}
\usepackage{amsmath}
\usepackage{amssymb}
\usepackage{amsthm}
\usepackage{nicefrac}
\usepackage{microtype}
\usepackage{graphicx}
\usepackage{capt-of}

\newcommand{\sd}[1]{\,{\scriptstyle\pm#1}}
\newcommand{\second}[1]{\underline{#1}}

\newtheorem{theorem}{Theorem}
\newtheorem{proposition}[theorem]{Proposition}

\newtheorem{remark}[theorem]{Remark}

\title{Semantic Latent Alignment for Flow Matching with Invertible Bi-Lipschitz Adapters}
\author{
  \textbf{Felix N\"utzel$^{1}$ \enspace
  Mischa Dombrowski$^{1}$ \enspace
  Luca Hagen$^{1}$ \enspace
  Bernhard Kainz$^{1,2}$} \\[0.5em]
  $^{1}$Friedrich-Alexander-Universit\"at Erlangen-N\"urnberg \\
  $^{2}$Imperial College London \\
}

\begin{document}

\maketitle
\lhead{Preprint}
\begin{abstract}
Pretrained latent spaces trade compression, reconstruction, and the learnability of the coordinates a downstream generator has to work in, which complicates adapting these spaces to different domains or tasks: continued training is expensive, and without a new alignment objective it often fails to improve downstream generation. We instead leave the encoder and decoder fixed and change only the coordinates between them, using an invertible bi-Lipschitz adapter aligned offline to a semantic teacher and frozen during class-conditional flow matching. Invertibility preserves compatibility with the pretrained autoencoder, while the bi-Lipschitz constraint limits how far alignment can distort its latent geometry. This provides a simple way to improve semantic organization while retaining the compression and reconstruction properties of the original latent space. On two medical imaging datasets, constrained latent adaptation improves feature metrics, seed stability and counterfactual edit strength, while remaining competitive on downstream classification. In contrast, unconstrained adaptation can become highly anisotropic and sensitive despite remaining invertible, highlighting the importance of controlling geometric distortion. Our results suggest that carefully constrained post-hoc adaptation offers a parameter-efficient and tuning-free alternative to jointly modifying the representation during generative training.
\end{abstract}

\section{Introduction}
\label{sec:intro}

Latent generative models, including latent diffusion~\citep{rombach2022high} and flow matching~\citep{lipman2023flow}, operate on compressed representations produced by pretrained variational autoencoders (VAEs)~\citep{kingma2013auto}. VAEs are typically optimized for a rate-distortion trade-off between compression and reconstruction fidelity~\citep{alemi2018fixing,bozkurt2021rate}. However, a representation that reconstructs well is not necessarily easy for a downstream model to learn~\citep{skorokhodov2025improving,zheng2026diffusion,dombrowski2026learnability}. Latent representation design therefore involves a third axis: \emph{downstream learnability}. We define learnability as the ease with which a bounded-capacity learner can recover task-relevant structure from a representation under the measurement process used by the downstream pipeline. It depends not only on which information the representation retains, but also on how that information is arranged~\citep{yao2025reconstruction,bfl2025representation}.

Fine-tuning the autoencoder can adapt the representation, but it also changes the decoder coordinates, requires a new latent cache, and adds a costly adaptation stage, with the adapted latents often still showing low learnability~\citep{dombrowski2026learnability}. Meanwhile, methods such as REPA~\citep{yu2025repa} align features inside the generative model during training. They improve the generator instead of the latent representation and repeat the alignment cost for each training run. 
We study a post-hoc alternative: adapt the latent coordinates once, then reuse them across flow-matching models.

We place a spatial invertible adapter between the frozen autoencoder and the generator. The adapter is trained offline on cached latents using features from a frozen semantic teacher. Flow matching consumes the adapted latents, while the adapter inverse returns generated samples to the coordinates expected by the pretrained decoder. Under exact inversion, the encoder-decoder reconstruction mapping remains unchanged. An unconstrained invertible map can still produce essentially collapsed latents, so we use a bi-Lipschitz design that bounds the distortion in both directions with a single scalar, estimated from the data.

For teacher alignment, we measure linear Centered Kernel Alignment (CKA)~\citep{kornblith2019similarity}, which only compares the relative similarity matrices of the two spaces. This matters because the adapted latent and the teacher embedding have different, very high dimensionalities and no shared coordinate system. We evaluate generation on two medical imaging datasets, using sample quality, semantic coverage, train-on-synthetic utility and counterfactual generation. Comparisons include the unchanged latent space, online generator alignment, autoencoder adaptation, and less constrained invertible maps.

Our contributions are:
\begin{enumerate}
    \item We propose a post-hoc latent adapter that is invertible and bi-Lipschitz by construction, aligned offline to a semantic teacher via CKA and frozen during flow matching. It leaves the autoencoder and the flow-matching objective unchanged, while a single scalar $\lambda$ bounds its distortion in both directions with a closed-form certificate.
    \item We make the adapter tuning-free: $\lambda$ can be estimated from a necessary condition on teacher triplets and the teacher from label geometry on training data, both fixed before any generator is trained. On both datasets the teacher choice agrees with the downstream results, and the method works across an eightfold range of $\lambda$.
    \item On our tested datasets, the adapter is competitive with online generator alignment and autoencoder adaptation, has the lowest seed variance in domain FID among the adaptation methods, and produces stronger counterfactual edits, at $405\times$ fewer parameters than REPA's projection head and a fraction of the compute of autoencoder fine-tuning.
    \item We show that invertibility alone is insufficient: invertible maps without Lipschitz control become nearly rank-deficient and fail as generators, and adapting only the clean endpoint instead of the whole flow path gives good but unreliable runs.
\end{enumerate}

\begin{figure*}[t]
    \centering
    \includegraphics[width=\linewidth]{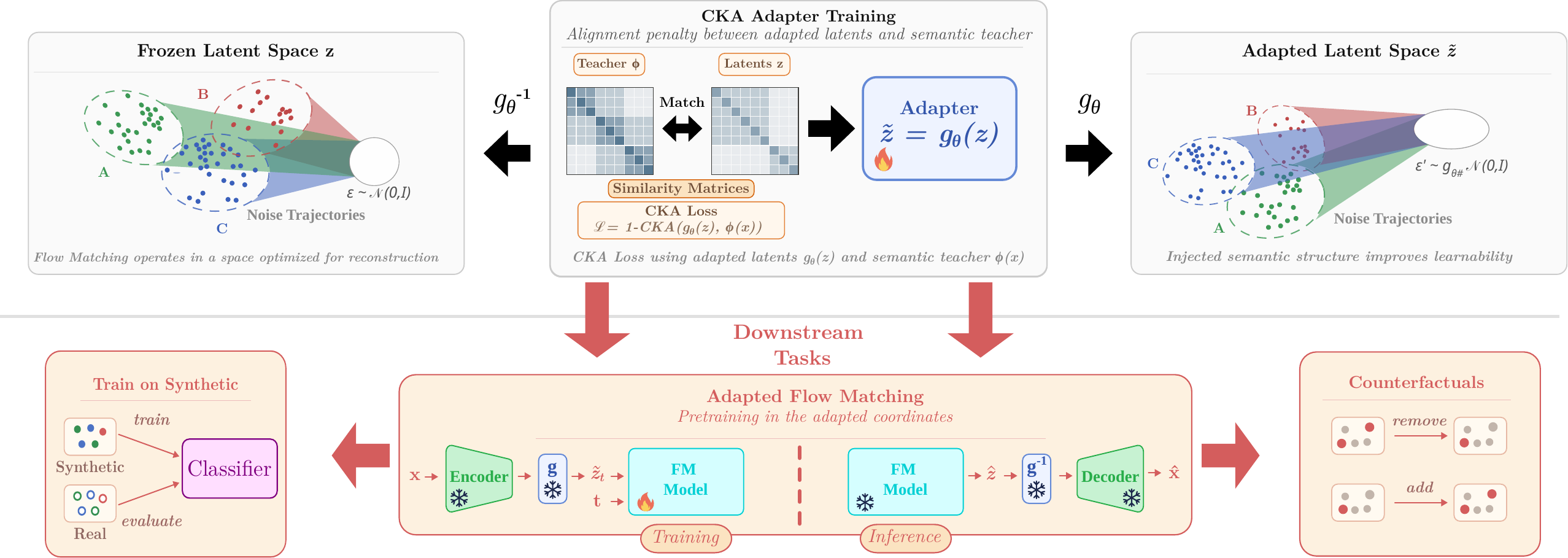}
    \caption{\textbf{Semantic latent adaptation.} \emph{Top:} an invertible adapter $g_\theta$ is trained offline on the cached latents $z$ of a frozen autoencoder, maximizing the linear CKA between the similarity structure of the adapted latents $g_\theta(z)$ and that of a frozen semantic teacher $\phi(x)$. The bi-Lipschitz constraint bounds how far this alignment can distort the original latent geometry, and $g_\theta^{-1}$ returns exactly to the coordinates the decoder expects. \emph{Bottom:} the adapter is then frozen, and flow matching is trained and sampled in the adapted space. The whole noise path is pushed through $g_\theta$, so the source distribution becomes $g_{\theta\#}\mathcal{N}(0,I)$, and generated latents are mapped back by $g_\theta^{-1}$ before decoding. We evaluate the generators by training classifiers on synthetic data and by counterfactual edits that add or remove a finding.}
    \label{fig:overview}
\end{figure*}

\section{Related work}
\label{sec:related}

Prior work improves latent generation by modifying the encoder-decoder pair, aligning representations inside the generator, or transforming an existing latent space. We study the latter via the additional requirement that the transformation remains invertible and well-conditioned.

\paragraph{Semantic alignment for latent generation.}
To semantically align the latents, \citet{yao2025reconstruction} and \citet{xu2025exploring} train visual tokenizers with semantic priors in VA-VAE and ReaLS, while in EQ-VAE,   the autoencoder is fine-tuned with an equivariance regularizer~\citep{kouzelis2025eqvae}.
Meanwhile, other semantic alignment procedures leave the VAE unchanged. One prominent example is REPA, which aligns intermediate denoiser states with features from a frozen vision encoder during generator training~\citep{yu2025repa}. \citet{singh2025matters} modify REPA to emphasize the spatial structure of the teacher features. 
Unlike prior alternatives, we leave  both the VAEs and the flow matching objective untouched. 

\paragraph{Invertible latent transformations.}
Invertible networks can be grouped into a few architectural families: coupling layers
\citep{dinh2015nice,dinh2017density,kingma2018glow}, autoregressive maps
\citep{kingma2016improving,papamakarios2017masked}, invertible residual blocks
\citep{behrmann2019invertible,chen2019residual}, and continuous-time flows
\citep{grathwohl2018scalable}. Most are trained as normalizing flows, where invertibility is what
makes the objective computable at all: maximum likelihood needs the change-of-variables term
$\log|\det J|$, so the map has to be invertible and its Jacobian determinant tractable. Applied to latents, this buys richer variational posteriors
\citep{rezende2015variational,kingma2016improving} or a density fitted to the codes of a pretrained autoencoder \citep{xiao2019generative,bohm2020probabilistic}. In each case the map transports one distribution onto another.

In our case, the invertible map is used as a reversible change of latent coordinates and not as a density model. The adapter aligns encoded latents with a semantic teacher, while its inverse returns generated latents to the coordinates expected by the frozen decoder. Because the objective does not require unrestricted transport, we constrain the map to remain close to an isometry and limit the distortion of the pretrained latent geometry.

To this end, we construct the adapter from invertible residual blocks
\citep{behrmann2019invertible,chen2019residual} with joint control of expansion and contraction. Related work certifies monotone residual layers with separate lower and upper Lipschitz bounds \citep{wang2024monotone}. Meanwhile, we tie them to expose a single distortion budget.

\section{Problem setup}
\label{sec:prelim}

\paragraph{Notation.}
We write $\|\cdot\|$ for the Euclidean norm. For a differentiable map $g:\mathbb{R}^m\to\mathbb{R}^m$
we write $J_g(z)$ for its Jacobian at $z$ and $\sigma_1(J_g(z))\ge\cdots\ge\sigma_m(J_g(z))\ge0$
for its singular values, with condition number $\kappa(J_g(z))=\sigma_1/\sigma_m$ when
$\sigma_m>0$. We reserve $\sigma$ for singular values and write $t$ for the flow-matching path
parameter. The pushforward of a measure $p$ under a measurable map $g$ is $g_\#p$.

\paragraph{Latents.}
Let $\mathcal{E}:\mathbb{R}^{H\times W\times3}\to\mathbb{R}^{C\times h\times w}$ be a pretrained
encoder and $\mathcal{D}$ its decoder. We identify the latent space $\mathcal{Z}$ with
$\mathbb{R}^m$, $m=Chw$, under the Euclidean norm of the flattened tensor, and cache
$\{(z_i,y_i)\}_{i=1}^N$ with $z_i=\mathcal{E}(x_i)$, with $y_i$ being a single label.

\paragraph{Semantic teacher.}
A frozen image encoder $\phi$ supplies the semantic reference. Our adapter objective uses its
global image embedding $\phi(x)\in\mathbb{R}^{d}$, a single vector per image, so it compares
images as a whole. The spatial alignment baselines (REPA, VA-VAE, ReaLS) instead use the encoder's
patch grid $\phi(x)\in\mathbb{R}^{d\times h'\times w'}$ and match it to latent or hidden-state
features location by location.

\paragraph{Bi-Lipschitz maps.}
A map $f:\mathbb{R}^m\to\mathbb{R}^m$ is $(\mu,\nu)$-bi-Lipschitz with $0<\mu\le\nu$ if
\begin{equation}
    \mu\|x-y\|\le\|f(x)-f(y)\|\le\nu\|x-y\|\qquad\forall\,x,y\in\mathbb{R}^m.
    \label{eq:bilip}
\end{equation}
The lower bound makes $f$ injective, while the upper bound limits
distortion. At $\mu=\nu=1$ the map is an isometry, and it is a near-isometry when both sit close
to $1$. A differentiable $(\mu,\nu)$-bi-Lipschitz map satisfies $\sigma_m(J_f(z))\ge\mu$ and
$\sigma_1(J_f(z))\le\nu$ at every $z$. The converse holds for a $C^1$ diffeomorphism of
$\mathbb{R}^m$ onto itself, where convexity of the domain lets one integrate the pointwise bounds
along segments.

\paragraph{Centered kernel alignment.}
To compare latents and semantic embeddings, we use CKA. For $n$ images, collect the adapted latents, each flattened to a vector by a map $P$ (Section~\ref{sec:objectives}), as rows of
$X\in\mathbb R^{n\times p}$, $X_i=P(g_\theta(z_i))$, and the teacher features as rows of
$Y\in\mathbb R^{n\times q}$, $Y_i=\phi(x_i)$, both column-centered. Their Gram matrices
$K=XX^\top$ and $L=YY^\top$ hold one entry per pair of images: $K_{ij}$ is how similar the
adapter finds images $i$ and $j$, and $L_{ij}$ is the teacher's verdict on the same pair.
Linear CKA is the cosine between these two similarity patterns,
\begin{equation}
    \operatorname{CKA}(X,Y)
    = \frac{\langle K,L\rangle_{\mathrm F}}{\|K\|_{\mathrm F}\,\|L\|_{\mathrm F}},
    \qquad
    \langle K,L\rangle_{\mathrm F}=\textstyle\sum_{i,j}K_{ij}L_{ij},
    \label{eq:cka}
\end{equation}
reaching $1$ when both spaces agree on every pair and $0$ when the two patterns are
unrelated.

\paragraph{Flow matching.}
Flow matching~\citep{lipman2023flow,liu2023flow} regresses a velocity field along a prescribed
probability path. For FLUX-type models the interpolant is $z_t=(1-t)z+t\varepsilon$ with
$\varepsilon\sim\mathcal{N}(0,I)$ and $t$ drawn from a flow-match scheduler, so $t=0$ recovers the
clean latent, and the regression target is $\varepsilon-z$. 

\section{Method}
\label{sec:method}

\subsection{Overview}
\label{sec:overview}

Our general approach is as specified in Figure~\ref{fig:overview}: First, we train an invertible adapter $g_\theta$ to semantically align the latent space with the  embeddings of a teacher $\phi$. For flow matching training, frozen $g_\theta$ adapts the cached original latents, while $g_\theta^{-1}$ carries generated latents back to the coordinates the decoder expects for inference:
\begin{equation}
    z=\mathcal E(x),\qquad \tilde z=g_\theta(z),\qquad x\approx\mathcal D(g_\theta^{-1}(\tilde z)).
    \label{eq:pipeline}
\end{equation}

The adapter leaves the latent dimensionality unchanged, and under exact inversion the autoencoder reconstruction mapping is identical to $\mathcal{D}\circ \mathcal{E}$. The adapter is bi-Lipschitz in the sense
of~\eqref{eq:bilip} with $\mu=e^{-\lambda}$ and $\nu=e^{\lambda}$, so it may rotate, bend and rescale the latent space, and it may not stretch or collapse any distance by more than $e^{\lambda}$. So, a single scalar $\lambda$ specifies how far the geometry is allowed to move.

\subsection{Aligning latents with a teacher}
\label{sec:objectives}

The semantic alignment between the cached latents and the adapter $g_\theta$ is achieved as follows: writing $P$ for the map that turns the spatial latent tensor into one vector per image, the primary objective is
\begin{equation}
    \mathcal L_{\mathrm{CKA}}
    = 1-\operatorname{CKA}\!\bigl(P(g_\theta(z)),\,\phi(x)\bigr).
    \label{eq:adapter_loss}
\end{equation}
We use two choices of $P$. Pool $0$ flattens the full adapted latent ($C\times h\times w=32\times64\times64$, i.e.\ $131{,}072$ dimensions), and pool $8$ first average-pools it to an $8\times8$ grid ($2{,}048$ dimensions), which discards fine spatial detail before the comparison. Linear CKA compares centered Gram matrices, so it reads relational structure and ignores isotropic scale. The objective therefore targets relational structure rather than latent magnitude, which suits spaces without shared coordinates.

\subsection{The distortion budget}
\label{sec:lambda}

Invertibility on its own is not enough, since an invertible map without additional bounds can still deform the original space so that important structure is destroyed, which makes it effectively unusable. Our unconstrained controls exhibit this failure mode: their estimated largest singular values exceed $360$ and their normalized participation ratios fall below $0.013$ (Table~\ref{tab:geometry}).

Therefore, similarly to \citet{wang2024monotone}, we build $g_\theta$ from residual blocks interleaved with Cayley-parameterized orthogonal mixers. Each block is $B_i(x)=c_ix+r_i(c_ix)$ with
$c_i=\cosh\lambda_i$ and $\operatorname{Lip}(r_i)\le\tanh\lambda_i<1$, where $r_i$ is a small spatial convolution stack whose Lipschitz bound comes from rescaling each of its weights to a fixed norm. The constant is therefore closed form, and it holds at every training step without a separate certification pass. The contractive-residual argument of
\citet{behrmann2019invertible} makes each block a bijection of $\mathbb{R}^m$ with bounds
$c_i(1\pm\tanh\lambda_i)=e^{\pm\lambda_i}$, and the orthogonal mixers are isometries, so the composition satisfies~\eqref{eq:bilip} with $\lambda=\sum_i\lambda_i$ and
\begin{equation}
    \mu=e^{-\lambda},\qquad \nu=e^{\lambda},\qquad \lambda\ge0 .
    \label{eq:reciprocal}
\end{equation}
Consequently, raising $\lambda$ buys larger contractions and expansions.
Propositions~\ref{prop:block} and~\ref{prop:compose} prove~\eqref{eq:reciprocal}; Appendix~\ref{app:arch} gives the architecture.

We take a closed-form bound instead of the tighter certificate of
\citet{wang2024monotone} because of what the map has to do. It acts on a $C\times h\times w$ tensor, so the bound has to compose over convolutions. It is inverted once per generated sample inside the sampling loop, so the inverse has to be a plain fixed-point iteration. Additionally, the budget is a constraint we set to preserve important properties of the original latents, so the bound being conservative does not matter to us.

\paragraph{Choosing $\lambda$ from the data.}
The budget has to be large enough for the adapter to reproduce the teacher's own relative distances. For a triplet that the teacher orders, write $\rho_z=d_z(i,j)/d_z(i,k)$ for the ratio in latent space and $\rho_\phi=d_\phi(i,j)/d_\phi(i,k)\le1$ for the ratio the teacher assigns.
Reaching the teacher's margin on that triplet needs
\begin{equation}
    \lambda\ \ge\ \tfrac12\max\bigl(0,\ \log\rho_z-\log\rho_\phi\bigr),
    \label{eq:lambda_lower}
\end{equation}
which is necessary locally per triplet.
(Proposition~\ref{prop:triplet} and Remark~\ref{rem:triplet}).
Sampling $1{,}500$ anchors with $20$ teacher neighbours and $20$ random candidates each, the
$99$th percentile of the right-hand side is roughly $0.5$ for both datasets, with under $0.01$ of movement across three sampling seeds. We adopt
$\lambda=0.5$ for both datasets, which covers $99.3\%$ and $98.5\%$ of sampled
triplets and gives $\mu=0.607$ and $\nu=1.649$. Section~\ref{sec:q3} sweeps
$\lambda\in\{0.25,0.5,1.0,2.0\}$ and shows that the method works across this range.

\subsection{Flow matching in the adapted coordinates}
\label{sec:paths}

Once the adapter is frozen, the generator must choose how to incorporate it into its
flow-matching path. To state the paths actually used in the experiments, let
$h_\theta=N\circ g_\theta$, where $N$ is the fixed preprocessing of
Section~\ref{sec:setup}. This notation records the generator's coordinates; $N$ is not learned by
our method and is not included in the adapter certificate. Write \(z\) for a clean latent,
\(\varepsilon\sim\mathcal N(0,I)\) for Gaussian noise, and \(t\) for the path parameter of
Section~\ref{sec:prelim}, with \(t=0\) clean and
\[
z_t=(1-t)z+t\varepsilon.
\]
Our main construction is the pushforward path
\begin{equation}
    \begin{aligned}
    x_t &= h_\theta(z_t), &
    u_t &= J_{h_\theta}(z_t)\,(\varepsilon-z),
    \end{aligned}
    \label{eq:pushforward}
\end{equation}
which transports the entire reference path through the frozen adapter and fixed preprocessing.
Its clean endpoint is therefore \(h_\theta(z)\), while its source is the pushed-forward Gaussian
\[
x_1=h_\theta(\varepsilon),
\qquad
x_1\sim h_{\theta\#}\mathcal N(0,I).
\]
Thus the source, target, and all intermediate states are expressed in the coordinates induced by
$h_\theta$, and the corresponding velocity is obtained by a forward-mode Jacobian--vector
product. Proposition~\ref{prop:velocity} bounds the adapter's contribution before the fixed preprocessing.

As an ablation, we also consider endpoint interpolation,
\begin{equation}
    \begin{aligned}
    x_t &= (1-t)\,h_\theta(z)+t\varepsilon, &
    u_t &= \varepsilon-h_\theta(z),
    \end{aligned}
    \label{eq:endpoint}
\end{equation}
which adapts only the clean endpoint while retaining the standard Gaussian source. Both define valid flow-matching objectives \citep{lipman2023flow} and share the same adapted target distribution at \(t=0\), but they induce different source distributions, intermediate geometries, and regression targets (see Table~\ref{tab:path_deployment} for an empirical comparison).

\section{Experiments}
\label{sec:exp}

\subsection{Setup and evaluation}
\label{sec:setup}

\paragraph{Datasets.} Our main dataset is a single-label filtered subset of MIMIC-CXR~\citep{johnson2019mimic}, consisting of eight disease classes split between 66,586 train samples, 8,261 validation samples and 8,361 test samples.
As alternative medical modality, we additionally report results for a single-label slices split of CT-RATE~\citep{hamamci2026generalist}, consisting of 13,309 train samples, 2,229 validation samples and 6,767 test samples across 15 disease classes.
Both datasets exhibit heavily long-tailed class distributions.

\paragraph{Metrics.} On these datasets, we measure the sample quality using Fréchet Inception Distance (FID)~\citep{heusel2018gan} and the diversity via the Image Retrieval Score (IRS)~\citep{dombrowski2025irs}. We report these measures in independent domain-specific feature spaces and domain-unspecific spaces, to capture both image realism and fidelity. The domain-specific extractors are a CheXNet-style DenseNet-121 classifier~\citep{huang2017densely,rajpurkar2017chexnet} for MIMIC-CXR and a DenseNet-121 classifier trained on the CT-RATE training split for CT-RATE. For general realism, we use DINOv2 features~\citep{oquab2023dinov2}.

For downstream utility, we train a DenseNet-121 on the synthetic images and evaluate it on the real test set (TSTR). Since the datasets are long-tailed, we report macro $F_1$ and balanced accuracy. For counterfactual edits, we use a TorchXRayVision~\citep{cohen2022xrv} DenseNet-121 as the oracle for edit strength. To measure minimality, we report LPIPS~\citep{zhang2018perceptual} to the factual and the change of the non-targeted findings, and we measure target realism by the FID to real images of the target class.

\paragraph{Training \& inference.} We use FLUX.2 latents for training the models, since \citet{dombrowski2026learnability} showed that they exhibit the best overall learnability among the representations they tested. The adapters are trained for $200$ epochs, which is a crude heuristic choice, since the loss usually saturates after only a couple of epochs. However, adapter matching is cheap: the adapter only consists of $36{,}871$ parameters (about $405$ times fewer than the REPA projection head), so even $200$ epochs take less compute than the autoencoder-adaptation baselines, which fine-tune the full autoencoder for $5$ (VA-VAE, full fine-tuning) or $10$ (ReaLS) epochs.
Before flow-matching training, every latent representation, including the unadapted baseline, is
centered and scaled per channel to standard deviation $0.5$ using statistics from its training
cache. This fixed affine preprocessing is inverted before applying $g_\theta^{-1}$ and decoding.

The FM models are trained on MIMIC-CXR for $10$ flow-matching epochs, while CT-RATE uses $30$ due to its lower number of samples. Each model is trained with seeds $0$, $1$ and $2$, without classifier-free guidance.
Every generator produces three seeded, class-balanced sets of $50{,}000$ images per FM seed, which are used for the feature metrics and as TSTR training sets.

\subsection{Main Results}
\label{sec:q1}

\begin{table*}[t]
    \centering
    \small
    \setlength{\tabcolsep}{4pt}
    \renewcommand{\arraystretch}{1.15}
    \resizebox{\textwidth}{!}{%
    \begin{tabular}{@{}l rr rr rr@{}}
        \toprule
        & \multicolumn{2}{c}{FID $\downarrow$} & \multicolumn{2}{c}{IRS $\uparrow$} & \multicolumn{2}{c}{TSTR} \\
        \cmidrule(lr){2-3}\cmidrule(lr){4-5}\cmidrule(l){6-7}
        \textbf{Method} & domain & DINOv2 & domain & DINOv2 & bAcc $\uparrow$ & macro $F_1$ $\uparrow$ \\
        \midrule
        \multicolumn{7}{@{}l}{\textsc{mimic-cxr}\;\; \textcolor{rulegray}{\scriptsize CXR-FID}} \\[1pt]
        \arrayrulecolor{rulegray}\cmidrule(r{2pt}){1-7}\arrayrulecolor{black}
        No adaptation & $20.16\sd{7.50}$ & $28.09\sd{13.74}$ & $0.315\sd{.103}$ & $0.412\sd{.115}$ & $0.557\sd{.021}$ & $0.168\sd{.046}$ \\[2pt]
        REPA~\citep{yu2025repa} & $19.34\sd{2.90}$ & $26.13\sd{9.16}$ & $0.331\sd{.060}$ & $0.430\sd{.111}$ & $\mathbf{0.590}\sd{.012}$ & $\mathbf{0.233}\sd{.010}$ \\[2pt]
        VA-VAE~\citep{yao2025reconstruction} & $\mathbf{15.05}\sd{2.92}$ & $\second{23.56}\sd{9.07}$ & $\mathbf{0.414}\sd{.063}$ & $\second{0.463}\sd{.089}$ & $0.563\sd{.024}$ & $0.179\sd{.032}$ \\
        ReaLS~\citep{xu2025exploring} & $19.05\sd{1.35}$ & $185.59\sd{6.90}$ & $0.207\sd{.018}$ & $0.121\sd{.002}$ & $0.544\sd{.005}$ & $0.094\sd{.014}$ \\
        Full VAE fine-tuning & $39.32\sd{34.49}$ & $441.29\sd{708.44}$ & $0.239\sd{.184}$ & $0.262\sd{.216}$ & $\second{0.570}\sd{.023}$ & $\second{0.210}\sd{.024}$ \\[3pt]
        \oursrow CKA adapter \textbf{(ours)} & $\second{17.87}\sd{.69}$ & $\mathbf{22.27}\sd{.84}$ & $\second{0.353}\sd{.010}$ & $\mathbf{0.478}\sd{.010}$ & $\second{0.570}\sd{.009}$ & $0.194\sd{.017}$ \\
        \midrule
        \multicolumn{7}{@{}l}{\textsc{ct-rate}\;\; \textcolor{rulegray}{\scriptsize CT-FID}} \\[1pt]
        \arrayrulecolor{rulegray}\cmidrule(r{2pt}){1-7}\arrayrulecolor{black}
        No adaptation & $20.23\sd{.16}$ & $186.9\sd{7.0}$ & $\second{0.865}\sd{.008}$ & $0.558\sd{.019}$ & $\mathbf{0.528}\sd{.011}$ & $0.084\sd{.013}$ \\[2pt]
        REPA~\citep{yu2025repa} & $\second{20.15}\sd{3.59}$ & $\second{143.8}\sd{66.6}$ & $\second{0.865}\sd{.032}$ & $\second{0.670}\sd{.201}$ & $\second{0.526}\sd{.017}$ & $0.071\sd{.025}$ \\[2pt]
        VA-VAE~\citep{yao2025reconstruction} & $20.67\sd{6.19}$ & $168.0\sd{83.4}$ & $0.863\sd{.066}$ & $0.636\sd{.182}$ & $\second{0.526}\sd{.011}$ & $\second{0.087}\sd{.028}$ \\
        ReaLS~\citep{xu2025exploring} & $97.78\sd{6.89}$ & $359.3\sd{36.0}$ & $0.448\sd{.012}$ & $0.355\sd{.048}$ & $0.521\sd{.005}$ & $0.051\sd{.009}$ \\
        Full VAE fine-tuning & $26.38\sd{2.56}$ & $245.7\sd{20.7}$ & $0.807\sd{.018}$ & $0.480\sd{.039}$ & $\second{0.526}\sd{.008}$ & $\mathbf{0.092}\sd{.018}$ \\[3pt]
        \oursrow CKA adapter \textbf{(ours)} & $\mathbf{15.46}\sd{1.38}$ & $\mathbf{111.9}\sd{39.2}$ & $\mathbf{0.890}\sd{.024}$ & $\mathbf{0.745}\sd{.094}$ & $\mathbf{0.528}\sd{.009}$ & $0.086\sd{.023}$ \\
        \bottomrule
    \end{tabular}}
    \caption{\textbf{Generation quality and train-on-synthetic (TSTR) utility.} Domain FID uses CheXNet features against the training set on MIMIC-CXR and a CT-trained DenseNet-121 against the test set on CT-RATE, and the DINOv2 columns use ViT-B/14 features. IRS is the image retrieval score normalized by that of the real test set. TSTR trains a DenseNet-121 on $50{,}000$ class-balanced synthetic images and evaluates it on the real test set. Mean and sample SD over three flow-matching seeds, with generation and classifier seeds averaged within each flow-matching seed; best in bold, second best underlined.}
    \label{tab:main}
\end{table*}

Table~\ref{tab:main} compares the adapter with the unchanged latent space, online alignment with REPA, and three ways of adapting the autoencoder itself. No method is best on every column, but the adapter is the only one that ranks first or second on nearly all of them on both datasets. On CT-RATE it is best on all four feature metrics, lowering CT-FID from $20.23$ to $15.46$ and DINOv2 FID from $186.9$ to $111.9$, and it matches the best balanced accuracy. On MIMIC-CXR it is best in the DINOv2 feature space and second in the domain feature space, where VA-VAE leads, while REPA gives the strongest train-on-synthetic scores. Neither advantage carries over to CT-RATE, where VA-VAE no longer improves CT-FID over the unadapted model and REPA falls below it in macro $F_1$. Adapting the autoencoder is also the least reliable option: ReaLS degrades DINOv2 FID on both datasets, and full fine-tuning diverged in one of three MIMIC-CXR runs.

On MIMIC-CXR, our adapter is also the most stable method: the SD of CXR-FID over flow-matching seeds is $0.69$, against $7.50$ without adaptation and about $2.9$ for REPA and VA-VAE; for DINOv2 FID it is $0.84$, against at least $6.90$ for every other method. Figure~\ref{fig:seed_stability} shows that the same holds relative to the mean, for sample quality as well as for the downstream metrics.

\subsection{Counterfactual controllability}
\label{sec:cf}

\begin{table*}[t]
\begin{minipage}[t]{0.45\textwidth}
\vspace{0pt}
    \centering
    \scriptsize
    \setlength{\tabcolsep}{1.8pt}
    \renewcommand{\arraystretch}{1.08}
    \resizebox{\linewidth}{!}{%
    \begin{tabular}{@{}lrrrr@{}}
        \toprule
        \textbf{Intervention} & $\Delta p$ $\uparrow$ & off-target $\downarrow$ & LPIPS $\downarrow$ & FID $\downarrow$ \\
        \midrule
        Add: no adapter & $.172\sd{.035}$ & $.149\sd{.028}$ & $.068\sd{.013}$ & $2.42\sd{.31}$ \\
        \oursrow Add: ours & $.192\sd{.017}$ & $.168\sd{.012}$ & $.086\sd{.010}$ & $2.22\sd{.12}$ \\
        \addlinespace[2pt]
        Remove: no adapter & $.098\sd{.016}$ & $.106\sd{.016}$ & $.078\sd{.014}$ & $.89\sd{.11}$ \\
        \oursrow Remove: ours & $.113\sd{.008}$ & $.121\sd{.005}$ & $.097\sd{.005}$ & $.74\sd{.05}$ \\
        \bottomrule
    \end{tabular}}
    \caption{\textbf{Counterfactual edits on MIMIC-CXR under the TXRV oracle.} Target-probability change, mean absolute change of the other findings against an edit to the factual's own class, LPIPS to the factual, and target-class FID in TXRV features; mean and sample SD over three flow-matching seeds. Our adapter makes a stronger, more target-realistic, but less selective edit.}\label{tab:cf_metrics}
\end{minipage}\hfill
\begin{minipage}[t]{0.52\textwidth}
\vspace{0pt}
    \centering
    \begin{tikzpicture}
    \begin{groupplot}[
      group style={group size=2 by 1, horizontal sep=1.0cm},
      width=0.51\linewidth, height=4.25cm,
      scaled ticks=false,
      yticklabel style={/pgf/number format/fixed, /pgf/number format/precision=2},
      xticklabel style={/pgf/number format/fixed, /pgf/number format/precision=1},
      ylabel={target $\Delta p$ $\uparrow$},
      grid=major, grid style={gray!18},
      tick label style={font=\scriptsize},
      label style={font=\small}, title style={font=\small},
      error bars/x dir=both, error bars/x explicit,
      error bars/y dir=both, error bars/y explicit,
      error bars/error bar style={line width=0.45pt},
    ]
    \nextgroupplot[
      title={Finding addition},
      xmin=1.75, xmax=2.90, ymin=0.12, ymax=0.235,
      legend to name=cflegend,
      legend columns=2,
      legend style={font=\scriptsize, draw=gray!40, fill=white, /tikz/every even column/.append style={column sep=6pt}},
    ]
    \addplot[only marks, mark=*, mark size=2.3pt, black]
      coordinates {(2.421425,0.171659) +- (0.311142,0.035369)};
    \addlegendentry{No adapter}
    \addplot[only marks, mark=square*, mark size=2.5pt, terracottadark]
      coordinates {(2.219908,0.191736) +- (0.115118,0.017251)};
    \addlegendentry{CKA adapter}
    \draw[-{Stealth[length=3pt]}, gray!65, thick]
      (axis cs:2.421425,0.171659) -- (axis cs:2.219908,0.191736);

    \nextgroupplot[
      title={Finding removal},
      xmin=0.52, xmax=1.18, ymin=0.055, ymax=0.145,
      ylabel={}, yticklabel pos=right,
    ]
    \addplot[only marks, mark=*, mark size=2.3pt, black]
      coordinates {(0.887410,0.098046) +- (0.106801,0.016022)};
    \addplot[only marks, mark=square*, mark size=2.5pt, terracottadark]
      coordinates {(0.744117,0.113400) +- (0.048591,0.007813)};
    \draw[-{Stealth[length=3pt]}, gray!65, thick]
      (axis cs:0.887410,0.098046) -- (axis cs:0.744117,0.113400);
    \end{groupplot}
    \node[font=\small, anchor=north] at
      ($(group c1r1.south)!0.5!(group c2r1.south)+(0,-0.48cm)$)
      {target-class FID $\downarrow$};
    \node[anchor=south] at
      ($(group c1r1.north)!0.5!(group c2r1.north)+(0,0.62cm)$)
      {\pgfplotslegendfromname{cflegend}};
    \end{tikzpicture}
    
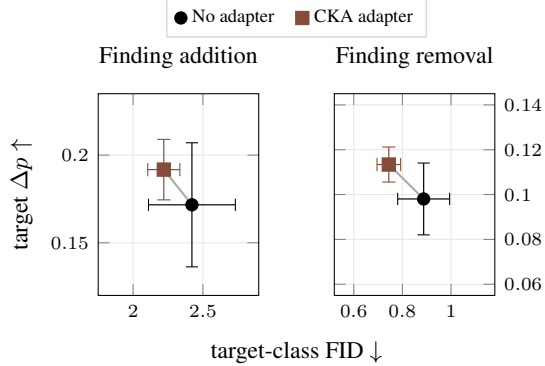
\captionof{figure}{\textbf{Effect and target realism of the edits in Table~\ref{tab:cf_metrics}.} Points show the means and bars one sample SD over three flow-matching seeds. The adapter increases target probability and lowers target-class FID in both directions.}\label{fig:counterfactual}
\end{minipage}
\end{table*}

To test counterfactual controllability, we generate counterfactuals following \citet{sanchez2022cf} and evaluate them following \citet{monteiro2023measuring}: we invert the factual latent along the learned flow, integrate it back under the target class and decode the result. Table~\ref{tab:cf_metrics} scores two edits on MIMIC-CXR, adding a finding to a No Finding radiograph and removing the finding from a diseased one, with a TorchXRayVision classifier that shares no training with the generator or the teacher. The adapter makes, on average, the stronger edit in both directions, raising the target probability by $0.192$ against $0.172$ for addition and by $0.113$ against $0.098$ for removal, and its edits lie closer to real images of the target class (Figure~\ref{fig:counterfactual}). The stronger edit also changes more of the image, with slightly higher LPIPS and off-target movement, so the adapter trades some selectivity for edit strength. As in generation, its spread across flow-matching seeds is about half that of the baseline.

\subsection{Teacher Selection}
\label{sec:q2}

The teacher decides which semantic structure the adapter imposes, so we fixed it before training any generator, from training data alone (Appendix~\ref{app:teacher}, Table~\ref{tab:teacher_selection}). MedImageInsight leads narrowly on MIMIC-CXR and MedSigLIP leads clearly on CT-RATE. The generators follow the size of these margins. On MIMIC-CXR, exchanging MedImageInsight for MedSigLIP keeps every column of Table~\ref{tab:path_matrix} within one seed SD. On CT-RATE, training against MedImageInsight instead (generation seed $0$) raises CT-FID from $15.46$ to $19.13$ and roughly halves macro $F_1$, from $0.086$ to $0.046$. Pooling the adapted latent to $8\times8$ before the CKA lowers the mean CXR-FID to $14.13$, but at three times the seed spread, a trade-off that recurs for the endpoint path below.


\begin{table*}[t]
\begin{minipage}[t]{0.49\textwidth}
\vspace{0pt}
        \centering
    \setlength{\tabcolsep}{3pt}
    \resizebox{\linewidth}{!}{%
    \begin{tabular}{@{}lcrrrr@{}}
        \toprule
        Teacher & Pool & FID $\downarrow$ & IRS $\uparrow$ & bAcc $\uparrow$ & macro $F_1$ $\uparrow$ \\
        \midrule
        MII & 8 & $\mathbf{14.13}\sd{2.00}$ & $\mathbf{.524}\sd{.046}$ & $.567\sd{.017}$ & $.183\sd{.043}$ \\
        MedSigLIP & 0 & $18.00\sd{.53}$ & $\second{.478}\sd{.021}$ & $\mathbf{.574}\sd{.016}$ & $\second{.185}\sd{.028}$ \\
        \midrule
        \oursrow MII & 0 & $\second{17.87}\sd{.69}$ & $\second{.478}\sd{.010}$ & $\second{.570}\sd{.009}$ & $\mathbf{.194}\sd{.017}$ \\
        \bottomrule
    \end{tabular}}
    \caption{\textbf{Teacher and pooling ablation on MIMIC-CXR.} MII is MedImageInsight~\citep{codella2024medimage}; pool $0$ applies CKA to the full flattened latent, pool $8$ to the latent average-pooled to $8\times8$. The configuration used everywhere else is shaded; best in bold, second best underlined; seeds aggregated as in Table~\ref{tab:main} (MedSigLIP: generation seed $0$ only). All rows use $\lambda=0.5$ and the pushforward path.}\label{tab:path_matrix}
\end{minipage}\hfill
\begin{minipage}[t]{0.48\textwidth}
\vspace{0pt}
        \centering
    \begin{tikzpicture}
    \begin{axis}[
        width=\linewidth,
        height=4.0cm,
        ybar,
        bar width=2.4pt,
        ymin=0,
        ymax=0.95,
        ylabel={seed CV $\downarrow$},
        symbolic x coords={CXR-FID,bAcc,macroF1},
        xtick=data,
        xticklabels={CXR-FID,bAcc,macro $F_1$},
        enlarge x limits=0.18,
        grid=major,
        grid style={gray!18},
        tick label style={font=\scriptsize},
        label style={font=\small},
        legend style={font=\tiny, at={(0.5,1.03)}, anchor=south, legend columns=2, draw=gray!40},
    ]
    \addplot[fill=black!20, draw=black!55] coordinates {(CXR-FID,0.372) (bAcc,0.038) (macroF1,0.272)};
    \addlegendentry{No adapter}
    \addplot[fill=gray!55, draw=gray!75] coordinates {(CXR-FID,0.150) (bAcc,0.020) (macroF1,0.042)};
    \addlegendentry{REPA}
    \addplot[fill=gray!35, draw=gray!60] coordinates {(CXR-FID,0.194) (bAcc,0.042) (macroF1,0.178)};
    \addlegendentry{VA-VAE}
    \addplot[fill=gray!75, draw=gray!90] coordinates {(CXR-FID,0.877) (bAcc,0.041) (macroF1,0.113)};
    \addlegendentry{Full VAE FT}
    \addplot[fill=terracotta!62, draw=terracottadark] coordinates {(CXR-FID,0.420) (bAcc,0.037) (macroF1,0.194)};
    \addlegendentry{CKA endpoint}
    \addplot[fill=terracottadark, draw=terracottadark] coordinates {(CXR-FID,0.039) (bAcc,0.016) (macroF1,0.089)};
    \addlegendentry{CKA pushforward}
    \end{axis}
    \end{tikzpicture}
    \captionof{figure}{\textbf{Relative variability over three MIMIC-CXR FM seeds.} Bars show sample SD divided by the mean. Pushforward has lower variability than endpoint, no adaptation, and both VAE baselines on all three metrics, while REPA has lower macro-$F_1$ variability.}\label{fig:seed_stability}
\end{minipage}
\end{table*}

\subsection{Ablations \& Analysis}
\label{sec:q3}

\begin{table*}[t]
\begin{minipage}[t]{0.45\textwidth}
\vspace{0pt}
    \centering
    \small
    \setlength{\tabcolsep}{1.5pt}
    \renewcommand{\arraystretch}{1.1}
    \begin{tabular}{@{}lccc@{}}
        \toprule
        \textbf{FM path} & CXR-FID $\downarrow$ & bAcc $\uparrow$ & $F_1$ $\uparrow$ \\
        \midrule
        Endpoint & $\mathbf{15.94}\sd{6.69}$ & $.565\sd{.021}$ & $.173\sd{.034}$ \\
        \oursrow Pushforward & $17.87\sd{.69}$ & $\mathbf{.570}\sd{.009}$ & $\mathbf{.194}\sd{.017}$ \\
        \bottomrule
    \end{tabular}
    \caption{\textbf{Flow-matching path on MIMIC-CXR.} The same frozen adapter, three flow-matching seeds (endpoint: generation seed $0$ only).}\label{tab:path_deployment}
\end{minipage}\hfill
\begin{minipage}[t]{0.51\textwidth}
\vspace{0pt}
    \centering
    \scriptsize
    \setlength{\tabcolsep}{3pt}
    \renewcommand{\arraystretch}{1.1}
    \begin{tabular}{@{}lcccc@{}}
        \toprule
        $\lambda$ & certified $[\mu,\nu]$ & FID $\downarrow$ & IRS $\uparrow$ & $F_1$ $\uparrow$ \\
        \midrule
        $0.25$ & $[.779,1.284]$ & $11.45$ & $.518$ & $.189$ \\
        $1.00$ & $[.368,2.718]$ & $17.22$ & $.477$ & $.137$ \\
        $2.00$ & $[.135,7.389]$ & $16.48$ & $.483$ & $.173$ \\
        \midrule
        \oursrow $0.50$ & $[.607,1.649]$ & $12.48$ & $.570$ & $.123$ \\
        \bottomrule
    \end{tabular}
    \caption{\textbf{Distortion budget on MIMIC-CXR.} MedImageInsight, pool $8$, flow-matching and generation seed $0$. $\lambda=0.5$ was fixed before the sweep.}\label{tab:lambda}
\end{minipage}
\end{table*}

\begin{table*}[t]
\begin{minipage}[t]{0.54\textwidth}
\vspace{0pt}

    \centering
    \scriptsize
    \setlength{\tabcolsep}{2.3pt}
    \renewcommand{\arraystretch}{1.04}
    \begin{tabular}{@{}llrrrr@{}}
        \toprule
        \textbf{Map} & \textbf{conditioner} & FID $\downarrow$ & IRS $\uparrow$ & bAcc $\uparrow$ & $F_1$ $\uparrow$ \\
        \midrule
        Additive & controlled & $19.63$ & $.415$ & $\mathbf{.582}$ & $\mathbf{.210}$ \\
        Affine & controlled & $\mathbf{17.13}$ & $\mathbf{.477}$ & $.576$ & $\mathbf{.210}$ \\
        RQ spline & controlled & $113.98$ & $.006$ & $.520$ & $.113$ \\
        Additive & free & $102.38$ & $.005$ & $.524$ & $.124$ \\
        Affine & free & $58.83$ & $.031$ & $.508$ & $.031$ \\
        RQ spline & free & $724.44$ & $.000$ & $.503$ & $.029$ \\
        \midrule
        \oursrow Residual & certified & $18.33$ & $.464$ & $.572$ & $.202$ \\
        \bottomrule
    \end{tabular}
    \caption{\textbf{Adapter family on MIMIC-CXR with the clean-only objective.} At flow-matching and generation seed 0 with three classifier seeds. We use invertible residual blocks~\citep{behrmann2019invertible,chen2019residual}, additive coupling~\citep{dinh2015nice}, affine coupling~\citep{dinh2017density,kingma2018glow}, and rational-quadratic splines~\citep{durkan2019neural}.}\label{tab:family}
\end{minipage}\hfill
\begin{minipage}[t]{0.43\textwidth}
\vspace{0pt}

    \centering
    \scriptsize
    \setlength{\tabcolsep}{2.0pt}
    \begin{tabular}{@{}lccc@{}}
        \toprule
        Map & bound & $\hat\sigma_{\max}$ & PR \\
        \midrule
        Certified residual & $[.607,1.649]$ & $1.042$--$1.056$ & $.9997$ \\
        Affine coupling, free & none & $363.2$ & $.0076$ \\
        Free residual & none & $1473.9$ & $.0121$ \\
        \bottomrule
    \end{tabular}
    \caption{\textbf{Local Jacobian diagnostics on eight validation anchors.} Only the certified residual has an analytic two-sided bound.}\label{tab:geometry}
    \medskip
    \centering
    \small
    \renewcommand{\arraystretch}{1.12}
    \begin{tabular*}{0.9\linewidth}{@{\extracolsep{\fill}}lr@{}}
        \toprule
        Measurement & Value \\
        \midrule
        Forward & $0.40$ ms \\
        Inverse & $0.69$ ms \\
        Round-trip error & $4.3\times10^{-6}$ \\
        Share of generation & $0.075\%$ \\
        \bottomrule
    \end{tabular*}
    \caption{\textbf{Per-sample adapter cost and numerical inversion error.}}\label{tab:inversion}
\end{minipage}
\end{table*}

\paragraph{Flow-matching path.} Endpoint interpolation reaches a lower mean CXR-FID than the pushforward path, but its SD over flow-matching seeds is nearly ten times larger, and it is lower on both downstream columns (Table~\ref{tab:path_deployment}). Keeping the source and every intermediate state in adapted coordinates is what makes the result repeatable, so we use the pushforward path throughout.

\paragraph{Distortion budget \& Inversion Cost.} Table~\ref{tab:lambda} sweeps $\lambda$ around the value estimated in Section~\ref{sec:lambda}. Every budget from $0.25$ to $2$ yields comparable results, with CXR-FID between $11.45$ and $17.22$ and no trend in the downstream scores beyond the seed spread of Table~\ref{tab:main}. The budget therefore needs no tuning beyond the triplet estimate. Inverting the adapter costs $0.69$ ms per sample ($0.075\%$ of generation) at a round-trip error of $4.3\times10^{-6}$ (Table~\ref{tab:inversion}).

\paragraph{Adapter family.} Table~\ref{tab:family} replaces our residual adapter with coupling flows. Couplings whose conditioner is norm-controlled perform on par with it, while the same couplings with a free conditioner fail as generators, with CXR-FID between $58.8$ and $724.4$ and IRS near zero, although they remain invertible. Table~\ref{tab:geometry} gives context for this: the free maps stretch some directions by more than $360$ and are close to rank-deficient, whereas the certified adapter stretches no direction by more than $6\%$. Controlled couplings therefore work as well, but only the residual adapter carries a certificate for the full map.

\section{Discussion}
\label{sec:discussion}

The results show that a small, frozen change of coordinates is competitive with methods that retrain the autoencoder or alter generator training, while being more stable across seeds than the other adaptation methods. Three properties make our approach practical: it has a single hyperparameter with a direct meaning, the distortion budget $\lambda$, which can be estimated from the data before training and which works across an eightfold range. The teacher can likewise be chosen from training data alone, and on both datasets this choice agreed with the downstream results. Finally, the adapter is cheap: it is fitted once in a few GPU-hours, reused by every flow-matching run and inverted in under a millisecond, and it leaves the autoencoder and the flow-matching loss untouched. The ablations indicate where the stability comes from. Invertibility alone is not enough, since invertible maps without Lipschitz control become nearly rank-deficient and fail as generators, and adapting only the clean endpoint gives good but unreliable runs. Bounding the distortion in both directions and keeping the whole path in the adapted coordinates makes the gains repeatable.

\paragraph{Limitations.}
We implemented the autoencoder-adaptation baselines with a medical teacher and their recommended settings, but more tuning could strengthen some of them, although needing it is itself a cost that can be avoided with our method. The adapter can only impose structure that its teacher carries, so it depends on a teacher with a good label geometry for the target domain, as the CT-RATE teacher comparison shows. Finally, the stronger counterfactual edits come with slightly larger off-target changes, so the adapter mainly improves edit strength.

\section{Conclusion}
\label{sec:conclusion}
We introduced an invertible bi-Lipschitz adapter that aligns pretrained latents with a semantic teacher offline and stays frozen during flow matching. On MIMIC-CXR and CT-RATE, it is competitive with generator alignment and autoencoder adaptation, more stable across seeds, and a small fraction of their cost. Its only hyperparameter can be estimated from the data, and its teacher can be selected before any generator is trained. Constrained post-hoc adaptation therefore offers a simple way to make a fixed latent space easier to learn, without retraining the autoencoder or changing the generative objective.

\subsection*{Acknowledgments}
The authors gratefully acknowledge the scientific support and HPC resources provided by the Erlangen National High Performance Computing Center (NHR@FAU) of the Friedrich-Alexander-Universit\"at Erlangen-N\"urnberg (FAU) under the NHR project b180dc. NHR funding is provided by federal and Bavarian state authorities. NHR@FAU hardware is partially funded by the German Research Foundation (DFG) -- 440719683.

\subsection*{Reproducibility statement}
In Section~\ref{sec:setup}, we detail the main setup needed to reproduce our results. Additionally, we provide our code in the supplementary material, with scripts to conveniently reproduce our environment and results, as well as tables used in this paper. The specific details behind our architecture is detailed in Appendix~\ref{app:arch}. Claims made in the main paper are supported by proofs in Appendix~\ref{app:proofs}.

\bibliography{iclr2027_conference}
\bibliographystyle{iclr2027_conference}

\appendix
\section{Adapter construction}
\label{app:arch}

\paragraph{Per-layer constants.}
Each residual branch ($r_i$) consists of ($1\times1$) pointwise and ($3\times3$) depthwise convolutions interleaved with SiLU activations. For pointwise convolutions, we rescale each weight matrix to a fixed Frobenius norm, which upper bounds its spectral norm. For depthwise convolutions, we rescale each kernel to a fixed ($\ell_1$) norm, which upper bounds the induced convolution operator norm by Young's inequality. Let ($L_{\mathrm{SiLU}}\approx 1.1$) denote the Lipschitz constant of SiLU~\citep{zhang2026general}. We normalize every hidden convolution to have operator norm at most ($1/L_{\mathrm{SiLU}}$); consequently, the composition of each hidden convolution with its subsequent SiLU activation is (1)-Lipschitz. The output convolution is instead normalized to have operator norm at most ($\tanh\lambda_i$). Multiplying these per-layer bounds therefore gives
$\operatorname{Lip}(r_i)\leq \tanh\lambda_i$.
This bound is enforced directly through the weight normalizations, requires no power iteration, and holds both at initialization and after every parameter update.

\paragraph{Bijectivity of $\mathbb{R}^m$.}
The singular-value characterization in Section~\ref{sec:prelim} requires a diffeomorphism of
$\mathbb{R}^m$ onto $\mathbb{R}^m$, not onto its image, so injectivity alone is not enough.
Contraction supplies the rest. For any target $y$ the map $u\mapsto y-r_i(u)$ contracts
$\mathbb{R}^m$ with factor $\tanh\lambda_i<1$, so the Banach fixed-point theorem gives a unique
$u$ with $u+r_i(u)=y$, and $x=u/c_i$ solves $B_i(x)=y$. Every block is therefore a bijection of
$\mathbb{R}^m$ onto itself. The Cayley-parameterized mixers are orthogonal and hence isometries of
$\mathbb{R}^m$ onto itself, so the composition inherits both invertibility and surjectivity.

\paragraph{Numerical inversion.}
We invert each block by the fixed-point iteration $b\leftarrow y-s\,r_i(b)$ implied by the same
contraction, capped at $50$ steps with a tolerance of $10^{-6}$, where $s$ is the block's gate.
Table~\ref{tab:inversion} gives the realized round-trip error and the per-sample latency.

\section{Proofs}
\label{app:proofs}

Each statement below is used in Section~\ref{sec:method} and proved here. Throughout,
$g_\theta$ is the adapter of Section~\ref{sec:lambda}.

\subsection{The budget of the composed adapter}

\begin{proposition}[Block bounds]
\label{prop:block}
Let $B(x)=cx+r(cx)$ with $c>0$ and $\operatorname{Lip}(r)\le\gamma<1$. Then $B$ is
$\bigl(c(1-\gamma),\,c(1+\gamma)\bigr)$-bi-Lipschitz. Taking $c=\cosh\lambda$ and
$\gamma=\tanh\lambda$ gives $B$ the bounds $\bigl(e^{-\lambda},e^{\lambda}\bigr)$.
\end{proposition}

\begin{proof}
Write $u=cx$ and $v=cy$, so $B(x)-B(y)=(u-v)+\bigl(r(u)-r(v)\bigr)$. The triangle inequality
and $\operatorname{Lip}(r)\le\gamma$ give
\[
\|B(x)-B(y)\|\le\|u-v\|+\gamma\|u-v\|=c(1+\gamma)\|x-y\|,
\]
and the reverse triangle inequality gives
\[
\|B(x)-B(y)\|\ge\|u-v\|-\gamma\|u-v\|=c(1-\gamma)\|x-y\| ,
\]
which is positive because $\gamma<1$. For the stated $c$ and $\gamma$,
\[
\cosh\lambda\,(1\pm\tanh\lambda)=\cosh\lambda\pm\sinh\lambda=e^{\pm\lambda}. \qedhere
\]
\end{proof}

\begin{proposition}[Composition, Eq.~\eqref{eq:reciprocal}]
\label{prop:compose}
Let $g_\theta=Q_{K+1}\circ B_K\circ\cdots\circ B_1\circ Q_1$, where each $B_i$ satisfies
Proposition~\ref{prop:block} with budget $\lambda_i\ge0$ and each $Q_i$ is orthogonal. Then
$g_\theta$ is $(e^{-\lambda},e^{\lambda})$-bi-Lipschitz with $\lambda=\sum_i\lambda_i$.
\end{proposition}

\begin{proof}
If $f_1$ is $(\mu_1,\nu_1)$-bi-Lipschitz and $f_2$ is $(\mu_2,\nu_2)$-bi-Lipschitz, then for all
$x,y$,
\[
\mu_2\mu_1\|x-y\|\le\mu_2\|f_1(x)-f_1(y)\|\le\|f_2(f_1(x))-f_2(f_1(y))\|\le\nu_2\nu_1\|x-y\| ,
\]
so the constants multiply. An orthogonal $Q$ satisfies $\|Qx-Qy\|=\|x-y\|$ and contributes the
factor $1$. Multiplying the block constants of Proposition~\ref{prop:block} gives
$\prod_i e^{\pm\lambda_i}=e^{\pm\sum_i\lambda_i}$.
\end{proof}

\begin{proposition}[Pushforward velocity, Eq.~\eqref{eq:pushforward}]
\label{prop:velocity}
Let $g_\theta$ be differentiable. Then for every $z_t$ and every $v\in\mathbb{R}^m$,
\[
e^{-\lambda}\|v\|\;\le\;\|J_{g_\theta}(z_t)\,v\|\;\le\;e^\lambda\|v\| .
\]
In particular the pushforward velocity $u_t=J_{g_\theta}(z_t)(\varepsilon-z)$ satisfies
$e^{-\lambda}\|\varepsilon-z\|\le\|u_t\|\le e^\lambda\|\varepsilon-z\|$.
\end{proposition}

\begin{proof}
A differentiable $(\mu,\nu)$-bi-Lipschitz map has $\sigma_m(J_{g_\theta}(z))\ge e^{-\lambda}$ and
$\sigma_1(J_{g_\theta}(z))\le e^{\lambda}$ at every $z$ (Section~\ref{sec:prelim}), so
$e^{-\lambda}\|v\|\le\|J_{g_\theta}(z_t)v\|\le e^{\lambda}\|v\|$, which is the claim.
\end{proof}

\subsection{The triplet budget}

\begin{proposition}[Necessary budget, Eq.~\eqref{eq:lambda_lower}]
\label{prop:triplet}
Fix a triplet $(i,j,k)$ and write $\rho_z=d_z(i,j)/d_z(i,k)$ for the latent distance ratio and
$\rho_\phi=d_\phi(i,j)/d_\phi(i,k)$ for the teacher's. If a $(e^{-\lambda},e^{\lambda})$-bi-Lipschitz
map $g_\theta$ achieves $d_{\tilde z}(i,j)/d_{\tilde z}(i,k)\le\rho_\phi$, then
\[
\lambda\;\ge\;\tfrac12\max\bigl(0,\;\log\rho_z-\log\rho_\phi\bigr).
\]
\end{proposition}

\begin{proof}
Write $\tilde z=g_\theta(z)$. Applying~\eqref{eq:bilip} to each pair,
$e^{-\lambda}d_z(i,j)\le d_{\tilde z}(i,j)\le e^{\lambda}d_z(i,j)$ and likewise for $(i,k)$.
Dividing the smallest possible numerator by the largest possible denominator bounds the adapted
ratio from below,
\[
\frac{d_{\tilde z}(i,j)}{d_{\tilde z}(i,k)}\;\ge\;\frac{e^{-\lambda}d_z(i,j)}{e^{\lambda}d_z(i,k)}
\;=\;e^{-2\lambda}\rho_z .
\]
The assumed inequality therefore forces $e^{-2\lambda}\rho_z\le\rho_\phi$, that is
$2\lambda\ge\log\rho_z-\log\rho_\phi$. Combining with $\lambda\ge0$ gives the statement.
\end{proof}

\begin{remark}
\label{rem:triplet}
Proposition~\ref{prop:triplet} is necessary and not sufficient. It constrains $\lambda$ from
below for one triplet, and it does not assert that any map inside the budget attains the ratio,
nor that a single map attains it on all triplets at once. We therefore read the quantiles of
Section~\ref{sec:lambda} as a lower envelope on a reasonable budget rather than as a feasibility
threshold, and we measure sensitivity to the choice directly in Section~\ref{sec:q3}.
\end{remark}


\makeatletter
\setlength{\@fptop}{0pt}
\setlength{\@fpsep}{12pt plus 2fil}
\setlength{\@fpbot}{0pt plus 1fil}
\makeatother
\section{Teacher selection and qualitative samples}
\label{app:teacher_qual}

\subsection{Teacher selection}
\label{app:teacher}

We fixed the teacher of each dataset before training any adapter or generator, using only the training split. For every candidate we drew the same $8{,}000$ single-label TRAIN images (seed $0$) and measured three properties of its embedding.
\emph{Label geometry} is the linear CKA between the embedding and the one-hot label matrix, reported relative to a $32\times32$ pixel thumbnail (the pixel floor) and relative to the flattened FLUX.2 latents that the adapter receives. The second ratio bounds how much label structure a CKA objective can add to the latents.
\emph{Retrieval} is the same-class rate among the $10$ nearest neighbours, excluding neighbours from the same patient (MIMIC-CXR) or the same volume (CT-RATE); it is scored on queries that carry a finding ($2{,}083$ and $5{,}191$ queries, random baselines $0.047$ and $0.100$).
\emph{Confound rate} is the fraction of those neighbours that come from the same patient or volume, where lower is better.
We selected the embedding that leads on label geometry and retrieval. Table~\ref{tab:teacher_selection} lists every candidate.

\begin{table}[!ht]
    \centering
    \scriptsize
    \setlength{\tabcolsep}{3.5pt}
    \begin{tabular}{@{}llrrrrr@{}}
        \toprule
        & Embedding & Label CKA & $\times$ pixels & $\times$ latents & Retrieval $\uparrow$ & Confound $\downarrow$ \\
        \midrule
        \multirow{5}{*}{\rotatebox{90}{\textsc{mimic}}}
        & Pixels ($32\times32$) & $0.023$ & $1.00$ & $0.71$ & $0.076$ & $0.002$ \\
        & MedSigLIP~\citep{sellergren2026medsiglip}, pooled & $0.121$ & $5.35$ & $3.78$ & $0.163$ & $0.016$ \\
        & MedSigLIP~\citep{sellergren2026medsiglip}, image embedding & $0.182$ & $8.01$ & $5.66$ & $0.209$ & $0.017$ \\
        & MedImageInsight~\citep{codella2024medimage}, pooled & $0.192$ & $8.46$ & $5.97$ & $0.279$ & $0.017$ \\
        & \ours{MedImageInsight~\citep{codella2024medimage}, image embedding} & \ours{$\mathbf{0.198}$} & \ours{$\mathbf{8.75}$} & \ours{$\mathbf{6.18}$} & \ours{$\mathbf{0.281}$} & \ours{$0.016$} \\
        \midrule
        \multirow{7}{*}{\rotatebox{90}{\textsc{ct-rate}}}
        & Pixels ($32\times32$) & $0.079$ & $1.00$ & $0.99$ & $0.130$ & $0.232$ \\
        & MedSigLIP~\citep{sellergren2026medsiglip}, pooled & $0.175$ & $2.20$ & $2.19$ & $0.198$ & $0.181$ \\
        & \ours{MedSigLIP~\citep{sellergren2026medsiglip}, image embedding} & \ours{$\mathbf{0.194}$} & \ours{$\mathbf{2.45}$} & \ours{$\mathbf{2.44}$} & \ours{$\mathbf{0.202}$} & \ours{$0.154$} \\
        & MedImageInsight~\citep{codella2024medimage}, pooled & $0.120$ & $1.52$ & $1.51$ & $0.197$ & $0.112$ \\
        & MedImageInsight~\citep{codella2024medimage}, image embedding & $0.116$ & $1.47$ & $1.46$ & $0.197$ & $\mathbf{0.106}$ \\
        & Curia~\citep{dancette2025curia} (CLS) & $0.137$ & $1.73$ & $1.72$ & $0.182$ & $0.263$ \\
        & FlexiCT-2D~\citep{li2026flexi} (CLS) & $0.095$ & $1.20$ & $1.19$ & $0.167$ & $0.214$ \\
        \bottomrule
    \end{tabular}
    \caption{Teacher selection on $8{,}000$ TRAIN images per dataset, shaded rows selected. Label CKA is linear CKA against the one-hot labels; the two ratios divide it by that of a $32\times32$ pixel thumbnail and of the flattened FLUX.2 latents. Retrieval is the same-class rate at $10$ neighbours excluding the same patient or volume, and confound is the same-patient or same-volume rate at $10$. The CT-native encoders Curia and FlexiCT-2D provide no image-text projection, so their class token stands in for the image embedding.}
    \label{tab:teacher_selection}
\end{table}

The two datasets select different teachers. On MIMIC-CXR MedImageInsight leads MedSigLIP by a small margin in label geometry ($6.18\times$ against $5.66\times$ the latents) and a large one in retrieval. On CT-RATE MedSigLIP leads every alternative in label geometry by a wide margin ($2.44\times$ against at most $1.72\times$), while the CT-native encoders carry more volume identity than the pixel floor. Section~\ref{sec:q2} compares these margins with the downstream results.

\subsection{Qualitative samples}
\label{app:qualitative}

Figures~\ref{fig:qual_mimic} and~\ref{fig:qual_ctrate} show uncurated samples from every method in Table~\ref{tab:main}. For each class we take the first generated image of that class from flow-matching seed $0$ and generation seed $0$. The top row shows the first TEST image of each class. No sample was filtered, re-drawn, or reordered.

\begin{figure*}[!ht]
    \centering
    \includegraphics[width=\textwidth]{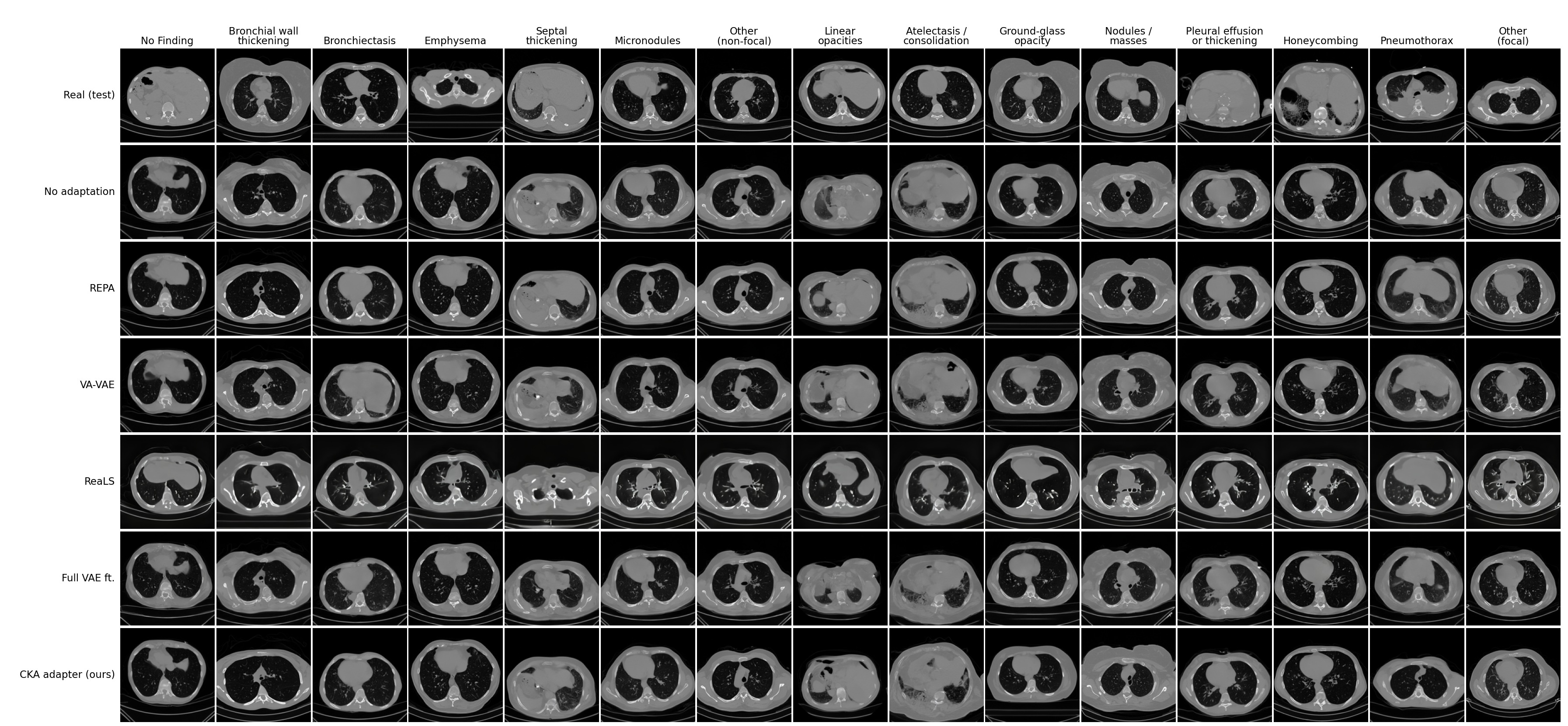}
    \caption{Uncurated CT-RATE samples, flow-matching seed $0$, generation seed $0$, for all fifteen classes: No Finding, six non-focal findings (bronchial wall thickening through other non-focal), and eight focal findings (linear opacities through other focal). Slices are shown with the posterior side down. Each column is one class; all generated entries of a column start from the same initial noise.}
    \label{fig:qual_ctrate}
    
\end{figure*}
\begin{figure*}[p]
    \centering
    \includegraphics[width=\textwidth]{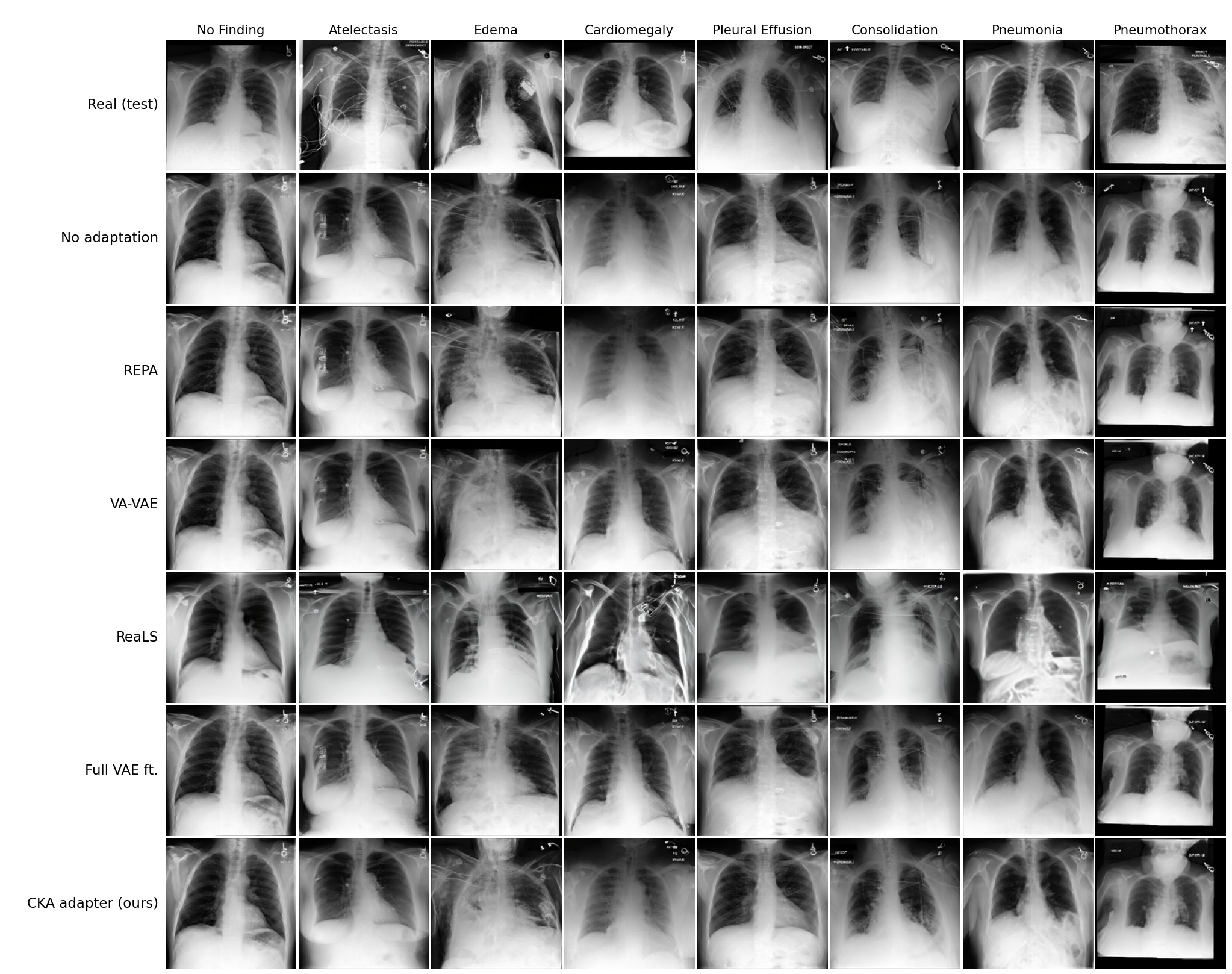}
    \caption{Uncurated MIMIC-CXR samples, flow-matching seed $0$, generation seed $0$. Each column is one class; all generated entries of a column start from the same initial noise.}
    \label{fig:qual_mimic}
\end{figure*}

\end{document}

%% file: math_commands.tex
\usepackage{amsmath,amsfonts,bm}

\def\eqref#1{equation~\ref{#1}}

\def\1{\bm{1}}

\DeclareMathAlphabet{\mathsfit}{\encodingdefault}{\sfdefault}{m}{sl}
\SetMathAlphabet{\mathsfit}{bold}{\encodingdefault}{\sfdefault}{bx}{n}

